\documentclass[10pt]{amsart}
\usepackage{egastyle}

\usepackage{graphicx}
\usepackage{subcaption}
\usepackage{caption}
\usepackage{tikz,tikzcd}
\usepackage{verbatim}
\usepackage{adjustbox}
\usepackage{multicol,xcolor}

\usepackage{tabularx}

\usepackage{algorithm}
\usepackage{algorithmic}

\usepackage{float}          % in preamble

\usepackage{xurl}
\DeclareUrlCommand\scriptname{\urlstyle{tt}}

\newcommand{\stg}{{\mathrm{stg}}}

\usepackage{amssymb,amsmath,amscd,amsthm,amsfonts,wasysym,mathrsfs,amsxtra}
\input xy
\xyoption{all}
\usepackage[hidelinks,breaklinks=true]{hyperref}

\begin{document}

%\twocolumn[\title{Phase II results}

\title{Wiring diagram extraction and gluing: a case study in classifying figure skating jumps using 3D dataset}

\author{Jason Lo}
\address{Department of Mathematics, California State University, Northridge, USA}
\email{jason.lo@csun.edu}

\author{Mohammadnima  Jafari}
\address{Department of Computer Science, California State University, Northridge, USA}
\email{mohammadnima.jafari@csun.edu}

\date{\today}

\begin{abstract}
Hasse clustering is an algorithm developed in \cite{LJ-D2WD} that extracts common patterns in sequential data and represents them in graphical forms.  As the number of expected clusters grows, however, the algorithm can become infeasible to run due to combinatorial complexity.  In this article, we describe a theory of gluing wiring diagrams, allowing iterative applications of Hasse clustering to achieve the same result as a single application.  We test our theory in the context of classifying videos of  figure skating jumps.
\end{abstract}

\maketitle

\tableofcontents

\begin{multicols}{2}

\section{Introduction}

\paragraph[The problem] Hasse clustering is an algorithm that extracts common patterns from sequences and represents the patterns as graphs called \emph{wiring diagrams} \cite[Algorithm 6.7]{LJ-D2WD}.  A wiring diagram can be thought of as a directed graph where each vertex is labeled with an abstract concept, and the arrows in the graph indicate the order relations among the concepts \cite{LoMR1}.  For example, the following wiring diagram, where the vertex labels are  movement characteristics shown by a figure skater, can act as a representation of the axel jump:

\begin{center}
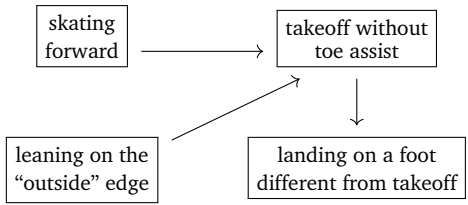

\begin{minipage}{\columnwidth}
\centering

\begin{tikzcd}
    {\fbox{\shortstack{\footnotesize skating\\ \footnotesize forward}}}
    &
    {\fbox{\shortstack{\footnotesize takeoff without\\ \footnotesize toe assist}}}
    \\
    {\fbox{\shortstack{\footnotesize leaning on the\\\footnotesize ``outside'' edge}}} 
    & 
       {\fbox{\shortstack{\footnotesize landing on a foot \\\footnotesize  different from takeoff}}}
    \arrow[from=1-1, to=1-2]
    \arrow[from=1-2, to=2-2]
    \arrow[from=2-1, to=1-2]
\end{tikzcd}

\captionof{figure}{Defining features of an axel jump.}
\label{fig:intro-axelWD}

\end{minipage}
\end{center}
When the number $r$ of distinct patterns within the data is large, or when the number of possible terms that appear in the sequences is large, the complexity of Hasse clustering grows combinatorially, making it unfeasible to run when hardware is limited to  a personal computer.

In this paper, we give two strategies for mitigating the limitations imposed by this combinatorial complexity.  First, we develop a theory of `gluing'   wiring diagrams that result from multiple iterations of  Hasse clustering.  Second, we give an updated version of \cite[Algorithm 6.7]{LJ-D2WD} that generates fewer Hasse diagrams along the way, thereby reducing the amount of RAM needed to run Hasse clustering.

\paragraph[Testing our theory]\label{para:intro-testingth} To test our theory of gluing wiring diagrams, we consider the problem of classifying jumps in figure skating.  In figure skating, there are 6 types of jumps: toe loop, flip, lutz, salchow, loop, and axel.  These jumps are distinguished by four pairs of technical features \cite{bgsu_aspire_jumps,grant_britannica_jumps}:
\begin{itemize}
    \item[(i)] Forward/backward direction: During the approach to the jump, is the skater facing forward or backward relative to the direction of travel?
    \item[(ii)] Inside/outside edge: During the approach to the jump, is the skater leaning on the ``inside'' or the ``outside'' edge of the blade that is in contact with the ice?
    \item[(iii)] Toe/edge jump: Does the skater use the ``free leg'' to tap the ice on their toes to assist with takeoff?  A jump is called a toe jump if such a move - called a `toe pick'-  occurs, and an edge jump otherwise.
    \item[(iv)] Same/different landing foot: Does the skater land on a foot that is the same, or different, from the foot that was last in contact with the ice before takeoff? 
\end{itemize}
An ideal jump involves the occurrence of one feature from each of the four pairs above, arranged in a certain configuration.  For example, Figure \ref{fig:intro-axelWD} shows the configuration for an axel.

In principle, if an autonomous system can recognize the aforementioned technical features of a jump from videos of skaters, it can translate each video into a sequence of features, and then Hasse clustering would be able to sort a collection of videos into clusters according to the detected configurations of the  features.  In other words, if an autonomous system already has the ability to detect lower-level concepts such as individual technical features of a jump, it would be able to identify and learn higher-level concepts (jumps) through Hasse clustering.

In testing our theory of gluing wiring diagrams, we used the FS-Jump3D dataset accompanying \cite{10.1145/3689061.3689077}.  This dataset contains 240 videos and corresponding 3D motion-capture data of 4 athletes performing all 6 types of jumps.  Running Hasse clustering on all 240 videos with $r=6$ (since we expect 6 clusters, one for each jump), however, is unfeasible.  As a result, we first applied Hasse clustering to the 240 videos using only three of the four pairs of technical features, leaving out the inside/outside edge pair; this allowed us to run Hasse clustering with $r=5$, with the downside being flip and lutz jumps were indistinguishable and formed a single cluster.  In the second iteration of Hasse clustering, we again used three pairs of technical features, but this time including the inside/outside edge pair; this allowed us to separate the flip-lutz combined cluster into two separate clusters that correspond to flip and lutz, respectively.  

We then glued the wiring diagrams  from the two iterations of Hasse clustering to recover the full wiring-diagram representations of flip and lutz jumps.  Overall, we produced 6 clusters from the 240 videos - with one cluster corresponding to each type of jump - and obtained the full wiring diagrams for flip and lutz.  (We could have applied more iterations of Hasse clustering to recover the full wiring diagrams of toe loop, salchow, loop, and axel as well, but it was not necessary since our primary goal was to sort the videos into 6 clusters and to demonstrate the process of gluing.)
 
\paragraph[Outline of this article] In Section \ref{sec:prelim-WDs}, we recall some basic definitions and results surrounding  wiring diagrams from \cite{LoMR1,LJ-D2WD}. In Section \ref{sec:HCupdated}, we give the pseudo-code for an updated version of Hasse clustering (Algorithm \ref{algo:HC-v2}), and prove a result on the absence of certain cycles in a graph that forms the backbone of the clustering process.  In Section \ref{sec:intro-FSjumps}, we explain in detail how different jumps in figure skating can be represented as wiring diagrams.  

In Section \ref{sec:HConFS}, we describe the process of applying iterative Hasse clustering to the FS-Jump3D  dataset and state the results.  In Section \ref{sec:theory-gluingWDs}, we give the theoretical foundation for gluing wiring diagrams from iterative Hasse clustering.  Theorem \ref{thm:main-01} is our main mathematical result - it says that under the right conditions, gluing together wiring diagrams from iterations of Hasse clustering produces the same result as applying Hasse clustering just once (if allowed by hardware).  

In Section \ref{sec:maineg}, we give an extended example that explains how the theory in Section \ref{sec:theory-gluingWDs} can be implemented in practice, and explain the outcome when the theory is applied to the results from our experiment in Section \ref{sec:HConFS}.  We end the article with Section \ref{sec:conclusion}, where we give a quick summary of this article followed by remarks on further applications of our framework of wiring diagrams.

\section{Wiring diagrams}\label{sec:prelim-WDs}

\paragraph[Graph theory - notation and terminology]  In this article, we will think of a directed graph $G$ as a quadruple $(V, A, s, t)$ where $V$ is the set of vertices, $A$ is the set of arrows, and $s, t : A \to V$ are the source and target functions, respectively, such that an arrow $a \in A$ points from $s(a)$ to $t(a)$.

\paragraph[Category theory]  A \emph{category} in mathematics is a collection of \emph{objects} and \emph{morphisms} that satisfies a list of axioms.  The reader may refer to references such as \cite{Awodey,SpivakCTS,maclane:71} for an introduction to category theory.  For this article, the reader may simply think of the objects of a category $\mathcal{O}$ as the elements of a set, and a morphism in $\mathcal{O}$ as an arrow that points from one object to another.  We will abuse notation and write $\mathcal{O}$ to denote the set of objects in $\mathcal{O}$. (We will always assume $\mathcal{O}$ is a `small' category.)  There are rules governing morphisms in a category, including:
\begin{enumerate}
    \item Whenever there is a morphism $f$ from object $x$ to object $y$, and a morphism $g$ from object $y$ to object $z$, we can compose these morphisms and obtain a morphism $gf$ from $x$ to $z$.
    \item   For every object $x$ in $\mathcal{O}$, there is an \emph{identity morphism} $1_x$ from $x$ to itself.  The identity morphisms satisfy  the following property: for every morphism $f$ in $\mathcal{O}$, say from $x$ to $y$, we have $f1_x = f = 1_yf$.
\end{enumerate}

The category of sets, denoted $\underline{\text{set}}$, is a category in which  objects are sets and  morphisms are functions.  The identity morphism $1_S$ for a set $S$ is simply the identity function from $S$ to itself.  Using category theory gives us a formal language for comparing directed graphs and the concepts they represent.

\paragraph[Wiring diagram graphs]  A \emph{wiring diagram}, as defined in \cite[Definition 2.1]{LJ-D2WD}, is a directed graph where each vertex is labeled with a sensor reading.  By encoding simpler concepts and the relations among them, wiring diagrams provide a framework for representing complex concepts \cite{LoMR1}. In this article, we adopt the following definition of a wiring diagram, which is slightly more general than \cite[Definition 2.1]{LJ-D2WD}.  Note that this definition uses the observation \cite[Remark 2.3]{LJ-D2WD} that the underlying graph of a wiring diagram is precisely a directed acyclic graph (DAG).

\begin{defn}[Wiring diagram]
Given a category $\mathcal{O}$, a \emph{wiring diagram} (WD)  over $\mathcal{O}$ is a quintuple $W = (V, A, s, t, L)$ where $(V, A, s, t)$ is a directed graph, called the \emph{underlying graph} of $W$, and $L$ is a function from $V$ to the set of objects in $\mathcal{O}$.  That is, a wiring diagram $W$ is a directed graph in which every vertex $v$ is labeled with an object $L(v)$ in $\mathcal{O}$.  A directed graph is called a \emph{wiring diagram graph} if it is the underlying graph of a wiring diagram.
\end{defn}

\subparagraph[How to read a wiring diagram] Given a wiring diagram $W$ as above, the label $L(v)$ at a vertex $v$ should be thought of as an event.  Given an arrow $a$ from vertex $x$ to vertex $y$ in $W$, we should think of the event $L(x)$ as having to occur before the event $L(y)$.  For example,  the wiring diagram in Figure \ref{fig:intro-axelWD} says that before taking off without a toe assist, the skater should be skating on a forward outside edge; then, the skater should land on the opposite foot from takeoff.  Altogether, these events with the described order make up the concept of an axel jump.

Following \cite[Definition 2.4]{LJ-D2WD}, we say a WD graph $G$ is \emph{quasi-skeleton} if it satisfies  two conditions:
\begin{itemize}
    \item[(a)] For any two distinct vertices $v, v'$ in $G$, there is at most one arrow from $v$ to $v'$.
    \item[(b)] For any two distinct vertices $v, v'$ in $G$, if there is a path of length at least 2 from $v$ to $v'$, then there is no arrow  from $v$ to $v'$.
\end{itemize}
The authors proved in \cite[Theorem 3.7]{LJ-D2WD} that a  graph is a quasi-skeleton WD graph if and only if it is the transitive reduction of a poset (i.e.\ a Hasse diagram).

\paragraph[The category of quasi-skeleton WD graphs] Quasi-skeleton WD graphs with the same set of vertices form a category \cite[Definition 4.2]{LJ-D2WD}, constructed as follows:  Suppose $J$ is a finite set.  For any quasi-skeleton WD graph $G$ whose set of vertices is $J$, we define the set
\[
 R'(G) := \{  (x,y) \in J \times J : \text{ there is an arrow } x \to y \text{ in } G \}.
\]
We then define $R(G)$ to be the set obtained by forcing transitivity and reflexivity on $R'(G)$.  Equivalently,
\begin{multline*}
R(G) = \{ (x,y) \in J \times J :  \text{ there is a path } x \to y \text{ in } G \} \\
\cup \{ (x,x) : x \in J\}.
\end{multline*}
Then we define   $\mathcal{R}(J)$ to be the category where the objects are the quasi-skeleton WD graphs whose set of vertices is $J$, and where there is a morphism $G_1 \to G_2$ whenever $R(G_2) \subseteq R(G_1)$.  

Intuitively, there is a morphism $G_1 \to G_2$ in the category $\mathcal{R}(J)$ whenever the WD graph $G_1$ ``generalizes'' to $G_2$, in the sense that $G_2$ contains less information than $G_1$, but all the information in $G_2$ is  consistent with that within $G_1$.

\paragraph[Operations on sequences and graphs]\label{para:SeeqGrOps}  We recall some more terminology and constructions in \cite{LJ-D2WD} that will be used in the rest of this article.

Let $X$ be a finite set.  If $s$ is a sequence in which all the terms are elements of $X$, then we say $s$ is a \emph{sequence in $X$}.  We say $s$ is a \emph{simple} sequence if  all the terms in the sequence are distinct.  We also say a wiring diagram $W=(V, A, s, t, L)$ over $\mathcal{O}$ is \emph{simple} if the function $L : V \to \mathcal{O}$ is injective, i.e.\ all the vertex labels in $W$ are distinct.

Given a finite set $X$, a sequence $s$ in $X$ and a subset $J$ of $X$, we write $m_J(s)$ to denote the sequence obtained from $s$ by omitting all the terms that are \emph{not} in $J$.

Given a set $V$ and a subset $S \subseteq V \times V$, we can associate to $S$ a directed graph $G$ with set of vertices $V$ and where there is an arrow $x \to y$ in $G$ if and only if $(x,y) \in S$.   Also, we say  a subset $P \subseteq V \times V$ is a poset if the relation $\preceq$ on $V$ defined by 
\[
x \preceq y \text{ if and only if } (x,y) \in P
\]
gives a partial order on $V$.  Note that whenever $P \subseteq V \times V$ is a poset, the graph associated to it is a DAG provided we ignore all the loops that arise from reflexivity of the partial order.

In Definition \ref{def:graphrestrictiontosubsest} below, by the `transitive reduction of a poset', we mean the transitive reduction (in the sense of \cite[Section 1]{aho1972transitive}; see also \cite[Definition 3.3]{LJ-D2WD}) of the graph associated to the poset.

\begin{defn}[restriction of a graph]\cite[Definition 5.13]{LJ-D2WD}\label{def:graphrestrictiontosubsest}
Suppose $G$ is a quasi-skeleton WD graph with set of vertices $V$.  For any subset $U \subseteq V$, we define the graph $G_U$ to be the transitive reduction of the poset
\[
    R(G) \cap (U \times U)
\]
and refer to it as the \emph{restriction of $G$ with respect to $U$}.
\end{defn}

As a special case of \cite[Definition 5.9]{LJ-D2WD}, given a simple sequence 
\[
s: s_1, s_2, \cdots, s_n
\]
in a finite set $X$, we write $\stg (s)$ to denote the path graph
\[
s_1 \to s_2 \to \cdots \to s_n
\]

The following is a slight modification of \cite[Definition 5.6]{LJ-D2WD}: 

\begin{defn}[consistent wiring diagram]
Let $s$ be a simple sequence in a finite set $X$, and $G$ a quasi-skeleton WD graph where all the vertices lie in $X$.  We say $s$ is \emph{consistent} with $G$, or that $G$ is \emph{consistent} with $s$ if, for any vertices $x, y$ in $G$ such that there is a path from $x$ to $y$ in $G$, both $x$ and $y$ appear in $s$ and $x$ appears earlier than $y$ in $s$.
\end{defn}

The notion of a flattening will be useful when we compare path WD graphs (e.g.\ $\stg (-)$ of simple sequences) to non-path WD graphs:

\begin{defn}\cite[Definition 5.18]{LJ-D2WD}\label{def:flattening-graph}
Suppose $G$ is a quasi-skeleton WD graph with set of vertices $V$.  We say a graph $P$ is a \emph{flattening} of $G$ if the following conditions are satisfied:
\begin{itemize}
    \item[(i)] $P$ is a path graph.
    \item[(ii)] The set of vertices of $P$ is also $V$.
    \item[(iii)] There is a morphism $P \to G$ in the category $\mathcal{R}(V)$.
\end{itemize}
\end{defn}
The intuition is that every morphism from a path quasi-skeleton WD graph $G'$ to a quasi-skeleton WD graph $G$ must factor through a flattening of $G$ (see \cite[Proposition 5.16, Theorem 5.23]{LJ-D2WD}).

\section{Hasse clustering: updated  algorithm}\label{sec:HCupdated}

In \cite[Algorithm 6.7]{LJ-D2WD}, the authors introduced an algorithm referred to as \emph{Hasse clustering} that extracts  patterns from data.  The input of the algorithm includes sequences, while the output consists of wiring diagrams that represent common themes within the data.  We present here Algorithm \ref{algo:HC-v2}, an updated and more streamlined version of Hasse clustering.

\end{multicols}

\begin{algorithm}[H]
    \raggedright
    \caption{An updated version of Hasse clustering \cite[Algorithm 6.7]{LJ-D2WD}}
    \label{algo:HC-v2}
    \textbf{Input}:   a finite set $X$, a collection of simple sequences $s^{(1)}, ..., s^{(p)}$ in $X$, a subset  $J=\{e_1, \cdots, e_m\}$  of $X$, a real number $t \in [0,100]$, and a positive integer $r$.  \\
    %\textbf{Parameter}: Optional list of parameters\\
    \textbf{Output}: a finite set $\mathcal{C} = \{C_1, \cdots, C_l\}$ such that: (i)  each $C_i$ is a nonempty set of size at most $r$; (ii)  each element of each $C_i$ is a quasi-skeleton WD graph with vertex set $J$.
\\
     \textbf{Notation}: Below, we will  write $\mathcal{R}(J)$ to denote the set of objects in that category.
    \begin{algorithmic}[1] %[1] enables line numbers
       \STATE For each $1 \leq i \leq p$, let $S_i$ be the set of all quasi-skeleton WD graphs $G$ with vertex set $J$ satisfying: 
       if there is a path from $e_a$ to $e_b$ in $G$ (where $a \neq b$), then $e_a, e_b$ must both appear in the sequence $s^{(i)}$ and $e_a$ must appear before $e_b$.  (In other words, $s^{(i)}$ is consistent with $G$.) \label{alg:HC-v2:line-1}  %[also rephrase this using the notation in D2WD like $m_J(-)$ etc.?]
       \STATE Set $S = \cup_{i=1}^p S_i$.
       \STATE Find all nonempty subsets $\overline{C}$ of $S$ such that: (i) $|\overline{C}| \leq r$; (ii) given any two distinct elements $G_1, G_2 \in \overline{C}$, there are no morphisms between $G_1$ and $G_2$ in $\mathcal{R}(J)$; (iii) the proportion of indices $i$ such that $S_i$ contains at least one element of $\overline{C}$ is at least $t\%$, i.e.\
       \[
       \frac{\left| \{ i \in \mathbb{Z} : 1 \leq i \leq p, S_i \cap \overline{C} \neq \varnothing \}\right|}{p} \geq t\% .
       \] \label{alg:HC-v2:line-3}
       \STATE Let $\mathcal{G}$ be a directed graph where the vertices are the subsets $\overline{C}$ identified in line \ref{alg:HC-v2:line-3}.  Given two distinct vertices $\overline{C}', \overline{C}''$ in $\mathcal{G}$, declare an arrow from $\overline{C}'$ to $\overline{C}''$ if, for every $G' \in \overline{C}'$, there exists some $G'' \in \overline{C}''$ such that there is a morphism from $G'$ to $G''$ in the category $\mathcal{R}(J)$. \label{alg:HC-v2:line-4}
       \STATE The output $\mathcal{C}$ will be the set of all vertices $\overline{C}$ in $\mathcal{G}$ such that there are no arrows going into $\overline{C}$ in $\mathcal{G}$. \label{alg:HC-v2:line-5}
    \end{algorithmic}
\end{algorithm}

\begin{multicols}{2}

\paragraph There are three main differences between this version of Hasse clustering (Algorithm \ref{algo:HC-v2}) and its predecessor \cite[Algorithm 6.7]{LJ-D2WD}:
\begin{enumerate}
    \item In the previous version, each element of the output set $\mathcal{C}$ is a collection of path matrices of DAGs.  In the current version, each element of the output $\mathcal{C}$ is a collection of DAGs. 
    \item In the previous version, the algorithm would generate \emph{all} the quasi-skeleton WD graphs   with vertex set $J$ and then, for each such graph $G$, identify the set of all sequences $s^{(i)}$ such that $G \in S_i$.  This approach leads to memory demand that grows exponentially as the size of $J$ increases (e.g.\ see the table in \cite[Section 6.8]{LJ-D2WD}); the current approach considers only quasi-skeleton WD graphs that are generalizations of the $\stg (s^{(i)})$, thus reducing  demand on memory.
    \item In this version, we require condition (ii) in line \ref{alg:HC-v2:line-3} of the algorithm, whereas this condition is absent in the previous version.  Informally, condition (ii) demands that we exclude any set $\overline{C}$ that contains redundant information: if $G_1, G_2$ are two objects in $\mathcal{R}(J)$ with a morphism  between them, say $\alpha: G_1 \to G_2$ and $G_1, G_2$ both lie in $\overline{C}$, then $R(G_2)$ is a subset of $R(G_1)$ and $\overline{C}$ would not appear as a vertex in the graph $\mathcal{G}$ constructed in line \ref{alg:HC-v2:line-4}.
\end{enumerate}

\emph{A priori}, we do not know if the directed graph $\mathcal{G}$ constructed in Algorithm \ref{algo:HC-v2} is acyclic.  Nonetheless, we can show that any directed cycle that exists must not go through vertices $\overline{C}$ with cardinality $r$ as a set.

\begin{lem}
In the graph $\mathcal{G}$ from line \ref{alg:HC-v2:line-4} of Algorithm \ref{algo:HC-v2}, a directed cycle cannot contain any vertex $\overline{C}$ with $|\overline{C}|=r$.
\end{lem}

\begin{proof}
Suppose $\mathcal{G}$ contains a cycle that passes through some $\overline{C}_0$ with $|\overline{C}_0|=r$. Since arrows in $\mathcal{G}$ are defined only between distinct vertices, there must be a sequence of vertices in $\mathcal{G}$
\[
\overline{C}_0, \overline{C}_1, \cdots, \overline{C}_m
\]
with $m\geq 2$, $\overline{C}_0 \neq \overline{C}_1$ and $\overline{C}_0 = \overline{C}_m$, such that for every $0 \leq i \leq m-1$ and every $G \in \overline{C}_i$, there exists  $G' \in \overline{C}_{i+1}$ such that there is a morphism $G \to G'$ in $\mathcal{R}(J)$.  

As a result, for every $G \in \overline{C}_0$ there is a sequence of morphisms in $\mathcal{R}(J)$
\begin{equation}\label{eq:GisDAG-01}
G = G_0 \to G_1 \to \cdots \to G_m
\end{equation}
where $G_i \in \overline{C}_i$ for all $0 \leq i \leq m$.  In particular, there is a morphism from $G$ to $G_m$ in $\mathcal{R}(J)$ while they both are elements of $\overline{C}_0=\overline{C}_m$.  By condition (ii) in line \ref{alg:HC-v2:line-3} of Algorithm \ref{algo:HC-v2}, $G_m$ must be equal to $G$, giving us
\begin{equation}\label{eq:GisDAG-02}
R(G) = R(G_m) \subseteq R(G_{m-1}) \subseteq \cdots  \subseteq R(G_0) = R(G);
\end{equation}
since $R(G)$ is a finite set, equality must hold throughout.  Since a poset has a unique transitive reduction, it follows that all the $G_i$ are equal and all the morphisms in \eqref{eq:GisDAG-01} are identity morphisms.  Since this holds for every $G \in \overline{C}_0$, it follows that  $\overline{C}_0$ is a subset of $\overline{C}_1$ and so $r=|\overline{C}_0| \leq |\overline{C}_1|$.  Since $\overline{C}_1$ is a vertex in $\mathcal{G}$, however, we have $|\overline{C}_1| \leq r$  which forces $|\overline{C}_1|=r$ and in turn $\overline{C}_0 = \overline{C}_1$, which is a contradiction.
\end{proof}

\section{A case study: jumps in figure skating}\label{sec:intro-FSjumps}

In figure skating, there are six types of jumps: axel, toe loop, loop, salchow, flip, and lutz.  Different jumps are distinguished by four pairs of technical features that were mentioned in Section \ref{para:intro-testingth}:
\begin{itemize}
    \item[(i)] forward/backward direction; 
    \item[(ii)] inside/outside edge;  
    \item[(iii)] toe/edge jumps; 
    \item[(iv)] same/different landing foot. 
\end{itemize}

We  summarize the presence or absence of these features in each jump in Table \ref{tab:jumps}.

\begin{center}
\footnotesize
\setlength{\tabcolsep}{3pt}
\begin{tabularx}{\columnwidth}{|l|>{\centering\arraybackslash}X
                                 |>{\centering\arraybackslash}X
                                 |>{\centering\arraybackslash}X
                                 |>{\centering\arraybackslash}X|}
\hline
\textbf{jump} & \textbf{\mbox{forward/}\newline \mbox{backward}}
              & \textbf{inside/\newline outside}
              & \textbf{toe/ \newline edge}
              & \textbf{landing\newline foot} \\
\hline
toe loop & backward & outside & toe & same \\ \hline
flip     & backward & inside & toe & different \\ \hline
lutz     & backward & outside & toe & different \\ \hline
salchow  & backward & inside & edge  & different \\ \hline
loop     & backward & outside & edge  & same \\ \hline
axel     & forward & outside & edge  & different \\ \hline
\end{tabularx}
\captionof{table}{Features of the six jumps in figure skating.}
\label{tab:jumps}
\end{center}

\paragraph \label{para:event_ordering} Since forward/backward direction and inside/outside edge occur before a toe pick (if there is one) in an ``ideal'' jump, the configuration of features for each of the six jumps can be represented as a wiring diagram.  Using the notation in Table \ref{tab:events}, the WD graph for each of the six jumps is listed in \eqref{eq:jumpWDs} below.

\begin{center}
\begin{tabular}{ll}
\hline
Feature & Symbol \\
 \hline
forward direction      & $e_1$ \\
backward direction     & $e_2$ \\
outside edge  & $e_3$ \\
inside edge      & $e_4$ \\
toe pick (toe jump)     & $e_5$ \\
no toe pick (edge jump)     & $e_6$ \\
same landing foot         & $e_7$ \\
different landing foot         & $e_8$ \\
\hline
\end{tabular}
\captionof{table}{Symbols for features in a figure skating jump.}
\label{tab:events}
\end{center}

\begin{equation}\label{eq:jumpWDs}
\xymatrix@R=0.8pc@C=1.6pc{
  \text{toe loop} & e_2 \ar[r] & e_5 \ar[r] & e_7 \\
   & e_3 \ar[ur] & & \\
   \text{flip} & e_2 \ar[r] & e_5 \ar[r] & e_8 \\
    & e_4 \ar[ur] & & \\
    \text{lutz} & e_2 \ar[r] & e_5 \ar[r] & e_8 \\
    & e_3 \ar[ur] & & \\
    \text{salchow} & e_2 \ar[r] & e_6 \ar[r] & e_8 \\
    & e_4 \ar[ur] & & \\
    \text{loop} & e_2 \ar[r] & e_6 \ar[r] & e_7 \\
    & e_3 \ar[ur] & & \\
    \text{axel} & e_1 \ar[r] & e_6 \ar[r] & e_8 \\
    & e_3 \ar[ur] & & \\
}
\end{equation}

In general, wiring diagrams can be used to detect the potential presence of a concept within data, or to classify the different concepts within data via Hasse clustering.  Below, we describe what detection and classification mean in the context of figure skating data.

\paragraph[Detection] Armed with the wiring diagram for each jump  as in \eqref{eq:jumpWDs}, one can  take any video containing a figure skating jump, extract a sequence of features, and compare the sequence with the wiring diagrams in \eqref{eq:jumpWDs} to detect the type of jump in the video.  This gives an example of using wiring diagrams to \emph{detect} behaviors.  In the language of \cite{LJ-D2WD}, the jump in a video can be identified by checking which wiring diagram in \eqref{eq:jumpWDs} has the extracted sequence of events as a `flattening' \cite[Definition 5.19]{LJ-D2WD}.

\paragraph[Classification] \label{para:WDinFS-classification} In theory, given a collection of videos of figure skaters  performing jumps of all six  types, if  specialist algorithms are available to determine whether each feature in Table \ref{tab:jumps} occurs, then one can produce a temporal sequence of such features from each video and feed all these sequences into  Algorithm \ref{algo:HC-v2}.  The algorithm is then expected to sort the sequences into six clusters, each corresponding to a 
wiring diagram in \eqref{eq:jumpWDs}.  
We test this idea in the next section.

\section{Extracting sequences from   figure skating videos}\label{sec:HConFS}

In order to use wiring diagrams for either classifying (sorting into clusters) or detecting jumps in figure skating videos, we first need to convert these videos into sequences of features.

In this work, we use   the   FS-Jump3D dataset created in \cite{10.1145/3689061.3689077}.  This dataset contains videos and corresponding 3D motion-capture data (in JSON format) from 12 viewpoints, for four skaters performing axel, flip, loop, lutz, salchow, and toe loop jumps, with  ten videos for each skater and each jump.  In total, there are 240 videos.    These videos can be converted to sequences in the finite set $X = \{ e_1, \cdots, e_8\}$ via a  process outlined below.

\paragraph[Low-level cue extractors]\label{para:lowlevelcueextrac} For each of the four features associated to each jump (features (i) through (iv) in Section \ref{para:intro-testingth}), we wrote a  low-level script to extract a broad set of numerical cues from the 3D motion-capture data.  For each jump feature, the name of the corresponding script and examples of  cues considered are listed in Table \ref{tab:jumpfeaturesdesc}.  These low-level scripts do not consider how these cues might be correlated - they only aim to extract a broad set of parameters for later processing.

\end{multicols}

\begin{table}[h]
\centering
\small

\begin{tabular}{|c|p{4.5cm}|p{6cm}|}
\hline
\textbf{Jump feature} & \textbf{Script name} & \textbf{Examples of 3D motion cues computed in script}  \\ \hline\hline
forward/backward  direction &  \scriptname{detect_forward_backward_from_c3d.py}  & Direction of the takeoff skate relative to travel; whether the heel or ankle leads;  torso and shoulder orientation relative to the skater's travel direction.   \\ \hline 
inside/outside edge & \scriptname{detect_edge_type_from_c3d.py} & Foot, ankle, knee, and hip geometry near takeoff;   the curvature, lean, and lateral position of the skater's approach path, normalized for body and limb scale.  \\ \hline
toe/edge jump & \scriptname{detect_toe_pick_from_c3d.py} & Whether a toe contacts the ice before takeoff; toe/heel motion and release timing. \\ \hline
same/different landing foot & \scriptname{detect_same_foot_from_c3d.py} & Which skate carries weight just before takeoff and which skate first becomes the  weight-bearing skate after landing. 
\\ \hline
\end{tabular}
\caption{Jump features and their corresponding low-level scripts.}
\label{tab:jumpfeaturesdesc}
\end{table}

\begin{table}[h]
\centering
\small
\begin{tabular}{|c|p{4.5cm}|p{3.2cm}|c|c|}
\hline
\textbf{Feature}
& \textbf{Model}
& \textbf{Profile}
& \textbf{20\% Test}
& \textbf{LOSO}
\\ \hline\hline

Forward/backward
& Logistic regression
& \texttt{takeoff\_core}
& $97.92\%$
& $99.58\%$
\\ \hline

Inside/outside
& L1 logistic regression
& \texttt{normalized\_full}
& $100.00\%$
& $100.00\%$
\\ \hline

Toe/edge
& Histogram gradient boosting
& \texttt{full}
& $95.83\%$
& $89.17\%$
\\ \hline

Same/different landing foot
& Histogram gradient boosting
& \texttt{aux\_fused}
& $91.67\%$
& $81.25\%$
\\ \hline
\end{tabular}
\caption{Model accuracies for 80/20 train-test split and LOSO.}
\label{tab:model-performance}
\end{table}

\begin{multicols}{2}

\paragraph[Classifiers]  For each of the four jump features, we then trained a classifier that detects that feature from a jump video and its 3D motion-capture data.  For example, the toe/edge classifier determines whether a toe pick occurs  (as in a toe jump) or does not occur (as in an edge jump) during the jump.  Different sets of cues from the low-level scripts in \ref{para:lowlevelcueextrac} were tested during training, to see which sets perform better in transfer across different skaters.  We refer to these sets of cues as  \emph{profiles}.  

The profile for a classifier might involve cues from a low-level script for a different jump feature.  For example,  the profile for inside/outside edge  consists of the full sets of cues from its corresponding low-level script, whereas the profile for same/different landing foot consists of cues from its corresponding low-level script as well as that for toe/edge.

In training the classifiers for forward/backward direction, toe/edge, and same/different landing foot, only 80\% of the 240 videos were used; the other 20\% were reserved for testing.  Once the optimal profiles for these three jump features were selected in building the classifiers, the classifiers were also retrained using video data from only three of the skaters; the classifiers were then tested on the videos of the fourth skater - we call this the \emph{leave-one-skater-out (LOSO)} scheme.  The LOSO scheme aims to show how robust our models are in transfer across different skaters.   For inside/outside edge, the same training schemes were run, but using \emph{only flip and lutz videos} - the reasons are explained below.   Results from accuracy tests for these classifiers are recorded in Table \ref{tab:model-performance}.  Figure \ref{fig:confusion} shows the confusion matrices for these classifiers.

A reason for using only flip and lutz videos for training the inside/outside edge classifier was that this jump feature is needed only when one tries to distinguish between flip and lutz jumps.  If we merge flip and lutz jumps into a single family, then all  five  jumps (toe loop, flip/lutz, salchow, loop, axel) can be distinguished from one another using just three features:  forward/backward, toe/edge, and same/different landing foot.

\paragraph[Motivation for iterative Hasse clustering] \label{para:motivation-iterHC}
The two main reasons for restricting the training data for the inside/outside edge classifier, however, are the following: (1) If we  run all four classifiers on all 240 videos,  we would obtain 240 sequences in the set $X = \{ e_1, \cdots, e_8\}$.  Since we expect these 240 sequences to be sorted into 6 clusters under Algorithm \ref{algo:HC-v2} according to the 6 jump types, we would choose $r=6$ and $J=X=\{e_1, \cdots, e_8\}$ as input for Algorithm \ref{algo:HC-v2}.  The  complexity of this algorithm, however, scales combinatorially as $r$ increases, making $r=6$ unfeasible to run on a personal computer.  (2) Distinguishing between inside and outside edge can be tricky, and this is a known difficulty in figure skating \cite{tanaka2023automatic}. We tried training the inside/outside edge classifier using 80\% of the videos for all six types of jumps, but the accuracy  was only 75\% in the 80/20 split test and 71.25\% for LOSO.  As a result, we could not reliably use the inside/outside edge classifier on all  240 videos in sorting them into clusters.

In order to keep $r \leq 5$ for the algorithm to run on our hardware, we ran clustering using  Algorithm \ref{algo:HC-v2} in two iterations: 
\begin{itemize}
    \item In the first iteration, we used only the forward/backward, toe/edge, and same/different landing foot classifiers, for which we expect five clusters based on the ideal WD graphs in \eqref{eq:jumpWDs}.  Flip and lutz jumps were expected to form one cluster since they are indistinguishable using only the three classifiers above.
    \item In the second iteration, we focused on the cluster containing flip and lutz jumps, then used the inside/outside edge classifier to sort the members into clusters which we expected to correspond to flip and lutz jumps.
\end{itemize}
This process of `iterative clustering'  motivated by the combinatorial complexity of  Algorithm \ref{algo:HC-v2} led to the question of how wiring diagrams produced in different stages can be pieced together to form a more comprehensive picture.  We develop a theory in Section \ref{sec:theory-gluingWDs} to give an answer to this question.

\paragraph[Constructing sequences of events] \label{para:eventseq_construction} Once the four classifiers were trained, they were applied to each of the 240 videos.  (The inside/outside edge classifier only has a 67.5\% accuracy when tested on all 240 videos, but this is not surprising since it was trained using only videos for flip and lutz jumps.)  Each video thus produced a collection of four events $e_i$, one from each of the pairs $\{ e_1, e_2\}, \{e_3, e_4\}, \{e_5, e_6\}, \{e_7, e_8\}$.  The four  events were then ordered using \emph{staged ordering}, whereby the order of  $e_1, e_2, e_3, e_4$ (forward/backward direction and inside/outside edge) is determined by the natural order in which they occur in the particular video, followed by $e_5$ or $e_6$ (toe/edge), and then $e_7$ or $e_8$ (same/different landing foot).  This choice of ordering is informed by our theory in \ref{para:event_ordering}.  The natural order among $e_1, e_2, e_3, e_4$ was determined using the video frames containing numerical cues feeding into the low-level scripts in Section \ref{para:lowlevelcueextrac}.  Overall, we obtained 240 sequences $s^{(i)}$ (where $1 \leq i \leq 240$) in the set $X=\{e_1, \cdots, e_8\}$.

\paragraph[Results from iterative Hasse clustering] \label{para:HCactualresults} 
We ran  Algorithm \ref{algo:HC-v2} (Hasse clustering)  in two iterations as outlined in \ref{para:motivation-iterHC}.  The results are described below.

\subparagraph[The first iteration] In the first run of Algorithm \ref{algo:HC-v2}, we used  the following input parameters: $X$ and $s^{(i)}$ for $1 \leq i \leq 240$ as in \ref{para:eventseq_construction}, $J=\{e_1, e_2, e_5, e_6, e_7, e_8\}, r=5$, and $t=95$.   Note that this   omits  the use of the inside/outside edge classifier.  

Algorithm \ref{algo:HC-v2}  returned an output $\mathcal{C}$ with only one element $\overline{C}$, which in turn has five elements.  These five elements (which are  quasi-skeleton WDs) are listed in Table \ref{tab:WDfrom1stHC} along with: 
\begin{itemize}
    \item size: the number of sequences $s^{(i)}$ that were sorted into that cluster, i.e.\ the number of indices $i$ such that $S_i$ contains that particular graph;
    \item \% total: `size' divided by 240, the total number of sequences, as a percentage;
    \item jump:  the jump whose `ideal' wiring diagram in \eqref{eq:jumpWDs} restricts to this particular quasi-skeleton WD graph with respect to the subset $J$ (see Definition \ref{def:graphrestrictiontosubsest} for the restriction of a graph);
    \item purity: the percentage of sequences in  this cluster for which the jump type is correctly identified.
\end{itemize}

\end{multicols}

\begin{center}
\begin{tabular}{ccccc}
\hline
graph & size & \% total & jump & purity \\
 \hline
$\xymatrix{e_1 \ar[r] & e_6 \ar[r] & e_8}$      & 40 & 16.7\% & axel & $40/40=100\%$ \\
$\xymatrix{e_2 \ar[r] & e_5 \ar[r] & e_7}$ & 41 & 17.2\% & toe loop & $39/41=95.1\%$   \\
$\xymatrix{e_2 \ar[r] & e_5 \ar[r] & e_8}$ & 76 & 31.7\% & flip/lutz & (see Table \ref{tab:WDfrom1stHC-2}) \\
$\xymatrix{e_2 \ar[r] & e_6\ar[r] & e_7}$ & 40 & 16.7\% & loop & $39/40=97.5\%$ \\
$\xymatrix{e_2 \ar[r] & e_6 \ar[r] & e_8}$ & 42 & 17.5\% & salchow & $40/42=95.2\%$\\
\hline
\end{tabular}
\captionof{table}{WD graphs produced by Hasse clustering using jump features other than inside/outside edge.}
\label{tab:WDfrom1stHC}
\end{center}

\begin{center}
\begin{tabular}{ccccc}
\hline
graph & size & \% total & jump & purity \\
 \hline
$\xymatrix@R=0.5em{e_2 \ar[r] & e_5 \\ e_3 \ar[ur] & }$      & 38 & 50.0\% & lutz & $38/38=100\%$ \\
$\xymatrix@R=0.5em{e_2 \ar[r] & e_5 \\ e_4 \ar[ur] &}$ & 38 & 50.0\% & flip & $38/38=100\%$\\
\hline
\end{tabular}
\captionof{table}{WD graphs produced by Hasse clustering using jump features other than same/different landing foot.}
\label{tab:WDfrom1stHC-2}
\end{center}

\begin{multicols}{2}

The cluster sizes in Table \ref{tab:WDfrom1stHC} sum to fewer than 240 because we chose a value of $t$ strictly less than 100.  This choice allowed for errors in the classifiers (which did not achieve perfect accuracy in the tests reported in Table \ref{tab:model-performance}) and irregularities in the input data for Algorithm \ref{algo:HC-v2} (e.g., if a skater performed a jump in an ``atypical'' fashion).  Consequently, sequences that were inconsistent with the common themes of the majority could remain unassigned to clusters.  Using $t=100$ would have required all 240 sequences to belong to a cluster, potentially causing the resulting wiring diagrams to be too general to convey useful information.

\subparagraph[The second iteration]  In the second run of Algorithm \ref{algo:HC-v2}, we used the same $X$ as in the first run, but with $J=\{ e_1, e_2, e_3, e_4, e_5, e_6\}$, $r=2$, and $t=100$.  The algorithm returned an output $\mathcal{C}$ with one element $\overline{C}$, which in turn has two elements.   The analogue of Table \ref{tab:WDfrom1stHC} for these two clusters is shown in Table \ref{tab:WDfrom1stHC-2}.  In this iteration, we chose $t=100$ because our inside/outside edge classifier had a 100\% accuracy on flip and lutz jumps (see Table \ref{tab:model-performance}).

\subparagraph[Discussion of results]  In the first iteration of Hasse clustering, 239 of the 240 videos in the FS-Jump3D dataset were sorted into 5 clusters.  (One of the 240 videos was not assigned to any cluster.)  We expected four of these clusters to correspond to   axel, toe loop, loop, and salchow, with the remaining  cluster corresponding to a mix of flip and lutz jumps.  Indeed, our algorithm returned the correct WD graphs for these clusters:  the graphs in  Table \ref{tab:WDfrom1stHC} are exactly the restrictions of the respective `ideal' versions in  \ref{eq:jumpWDs} with respect to the subset $\{e_1, e_2, e_5, e_6, e_7, e_8\}$.

In the second iteration of Hasse clustering, the flip/lutz cluster from the first iteration was further divided into two subclusters.  The expectation was that these two subclusters would correspond to flip and lutz, respectively.  Indeed, the graphs in Table \ref{tab:WDfrom1stHC-2} are the restrictions of the `ideal' versions  in  \ref{eq:jumpWDs} with respect to the subset $\{e_1, e_2, e_3, e_4,  e_5, e_6\}$.

Figure \ref{fig:pipeline} summarizes the pipeline and results of our method  in the case of figure skating data.

\begin{figure*}[p]
    \centering

    \includegraphics[
        width=\textwidth,
        height=0.39\textheight,
        keepaspectratio]{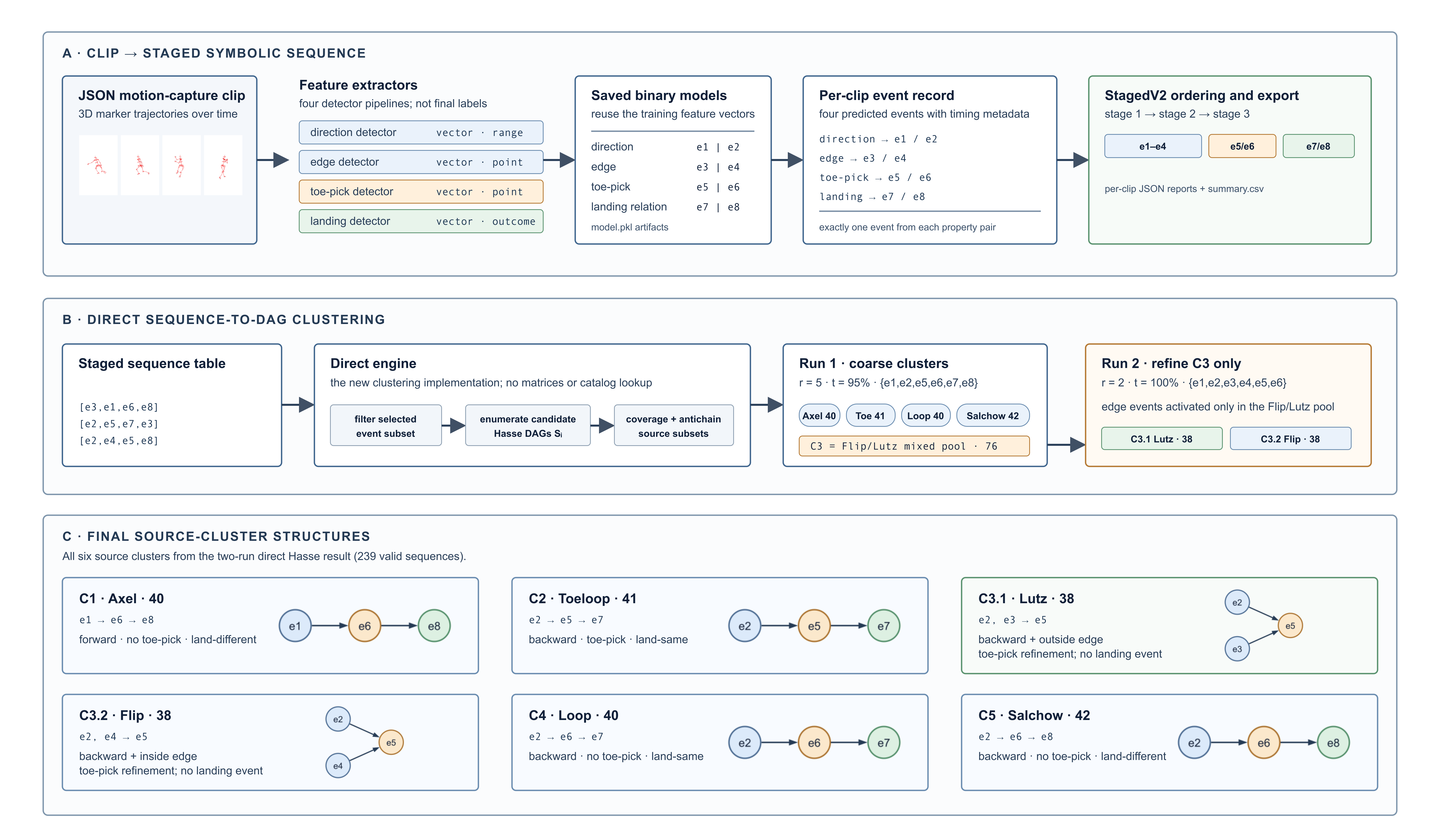}

    \caption{The complete pipeline of our method in three steps.  Part A illustrates the steps that convert videos to sequences.  Part B shows the process of Hasse clustering, while Part C shows the wiring diagram graphs that result from the two stages of clustering.}
    \label{fig:pipeline}

    \vspace{1.5em}

    \includegraphics[
        width=\textwidth,
        height=0.45\textheight,
        keepaspectratio
    ]{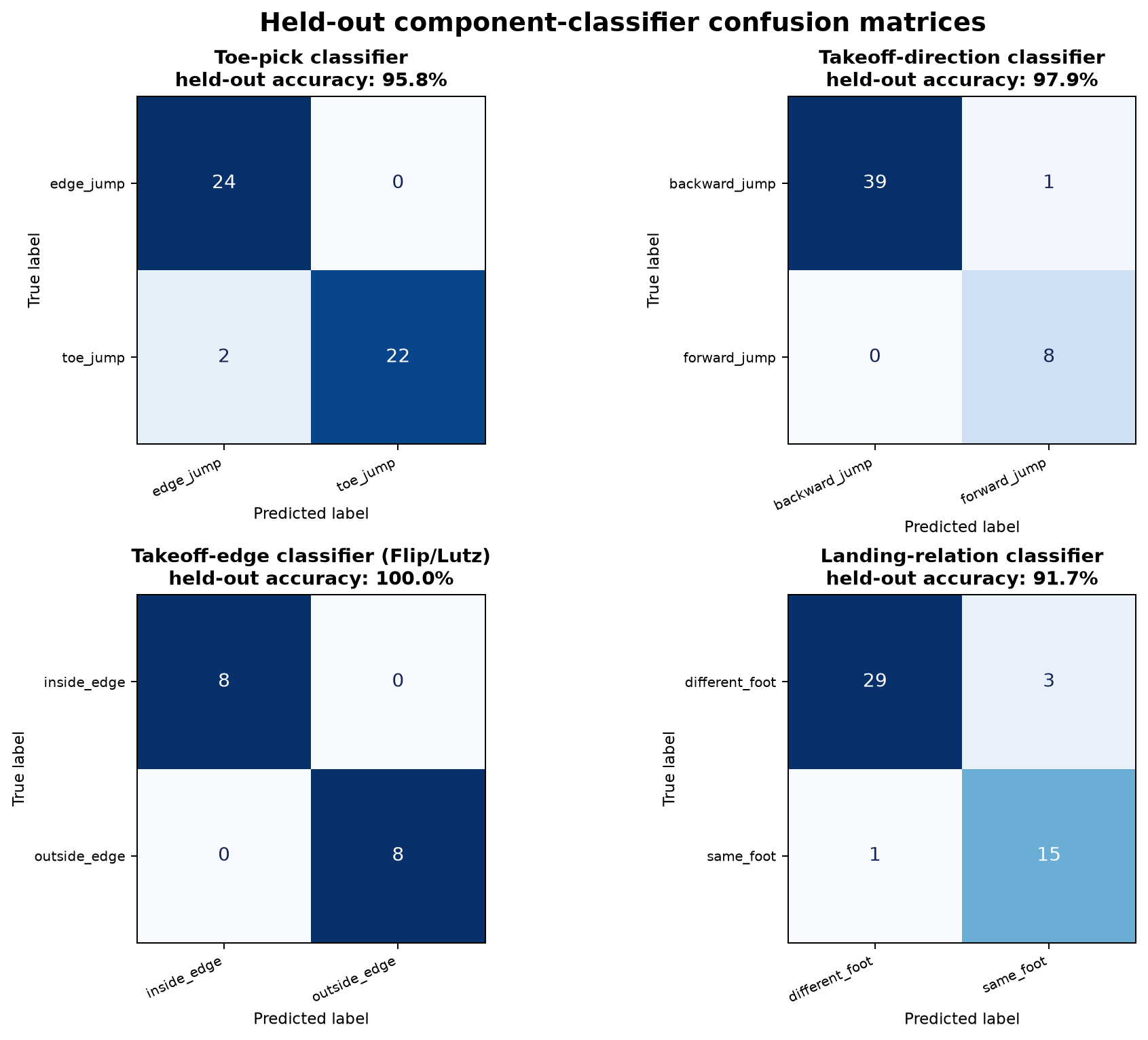}

    \caption{Confusion matrices for the four classifiers.}
    \label{fig:confusion}
\end{figure*}

In Section \ref{sec:theory-gluingWDs}, we outline a theory of how WD graphs from multiple iterations of Hasse clustering can be combined to give larger graphs that represent the subclusters.  We also give a theoretical example in Example \ref{eg:main-01} that illustrates the theory for data under ``perfect  conditions'', and compare it with results from our experiment in this section.

\section{A theory of gluing wiring diagrams}\label{sec:theory-gluingWDs}

In this paper, we analyzed data from figure skating by performing two iterations of Hasse clustering.   In the first round, we refrained from using the inside/outside edge classifier; in the second round, we applied the inside/outside edge classifier to data on which it was specifically trained and for which it achieved  an accuracy rate of 100\% (see Table \ref{tab:model-performance}).  This motivates the following theoretical question:

\begin{center}
\fbox{\parbox{0.9\linewidth}{\centering
\emph{If we perform Hasse clustering multiple times using different subsets of events, how can we relate the results to those obtained  from performing Hasse clustering once using the full set of events?}
}}
\end{center}

In the rest of this section, we present some theory that answers this question in a special case that covers the figure skating data we studied in Section \ref{sec:HConFS}.

We begin with a lemma that says that when we  set $r=1$ in Algorithm \ref{algo:HC-v2} with $t=100$, the algorithm returns only one possible graph.  Intuitively, this says that if we want to summarize all the sequences in a single graph, there is only one way to do so.

\begin{lem}\label{lem:commonWDG}
Let $X$ be a finite set, and $s^{(1)}, \cdots, s^{(p)}$ simple sequences in $X$, and  $J$  a subset of $X$. Suppose  applying Algorithm \ref{algo:HC-v2} with $r=1$ and $t=100$  to the sequences $s^{(1)}, \cdots, s^{(p)}$ produces the output $\mathcal{C} = \{ \{ G_1\}, \cdots, \{G_n\}\}$ where each $G_i$ lies in $\mathcal{R}(J)$.  Then %whenever $\mathcal{C}$ is nonempty, we have 
$n=1$.
\end{lem}

\begin{proof}

Let $G^\ast$ denote the transitive reduction of the graph with vertex set $J$ such that  there is a path from vertex $a$ to vertex $b$ in $G^\ast$ if and only if, for all $i$, the terms $a, b$ both appear in $s^{(i)}$ and $a$ appears before $b$.  By construction, $\{ G^\ast \}$ must be one of the vertices in the graph $\mathcal{G}$ from line \ref{alg:HC-v2:line-4} of Algorithm \ref{algo:HC-v2}.

Since $t=100$, for any  $1 \leq j \leq n$  there is a path from vertex $x$ to vertex $y$ in $G_j$  only if the following is true: for all $i$, both $x, y$ appear in the sequence $s^{(i)}$ and $x$ appears before $y$.  This means $R(G_j) \subseteq R(G^\ast)$ for all $j$, i.e.\ there is a morphism from $G^\ast$ to $G_j$ in the category $\mathcal{R}(J)$ for all $j$.  From line \ref{alg:HC-v2:line-5} of   Algorithm \ref{algo:HC-v2}, however, there cannot be any arrow pointing to $\{ G_j \}$ in the graph $\mathcal{G}$ for any $j$.  Since $\{ G^\ast \}$ is a vertex in $\mathcal{G}$ from the previous paragraph, this means  all the $G_j$ must be equal to $G^\ast$, and so $n=1$ as claimed.
\end{proof}

\begin{defn}[common WD graph]\label{def:commonWDG}
Under the hypotheses of Lemma \ref{lem:commonWDG}, Algorithm  \ref{algo:HC-v2} produces the output $\mathcal{C} = \{ \{ G^\ast \}\}$.  We will refer to the graph  $G^\ast$  as the \emph{common wiring diagram graph} of the sequences $s^{(1)}, \cdots, s^{(p)}$ with respect to $J$. We take $J$ to be $X$ if $J$ is not specified.
\end{defn}

\paragraph \label{para:pathglueingcond} Let  $G$ be a quasi-skeleton WD graph with vertex set $X$.  Suppose $V_1, \cdots, V_n$ are subsets of $X$ satisfying the gluing condition
\begin{itemize}
\item[(U)] every path $p$ in $G$ can be written as a concatenation
\[
p = p_1 . \cdots . p_k
\]
where each $p_i$ is a path in some $G_{V_{i'}}$ for some $1 \leq i' \leq n$.
\end{itemize}

\begin{prop}\label{prop:WDgluing-01}
Let $G$ be a quasi-skeleton WD graph  with vertex set $X$.  If  $V_1, \cdots, V_n$ are subsets of $X$ satisfying  condition (U), then $G$ is the transitive reduction of 
\[
  \bigcup_{i=1}^{n} R(G_{V_i}).
\]
\end{prop}

Here, by a transitive reduction of a subset of $V_i \times V_i$, we mean the transitive reduction of the DAG associated to that subset of $V_i \times V_i$ as defined in \ref{para:SeeqGrOps}.

\begin{proof}
Note that there is a 1-1 correspondence between transitively closed DAGs and their transitive reductions \cite[Section 1]{aho1972transitive}.  Since  $G$ is a quasi-skeleton WD graph,  and hence the transitive reduction of a DAG by \cite[Theorem 3.7]{LJ-D2WD}, it suffices to show
\begin{equation}\label{eq:glue-01}
R(G) = \text{ transitive closure of }\bigcup_{i=1}^{n} R(G_{V_i}).
\end{equation}

First, we prove the inclusion $\supseteq$ in \eqref{eq:glue-01}.  Suppose $(x,y)$ lies in the right-hand side of \eqref{eq:glue-01} and $x \neq y$.  Then there is a sequence of ordered pairs in $X \times X$
\[
(a_0,a_1), (a_1, a_2), \cdots,  (a_{m-1}, a_m)
\]
with $m \geq 1, a_0=x, a_m=y$ such that each $(a_{i-1}, a_i)$ lies in $R(G_{V_{i'}})$ for some $1 \leq i' \leq n$.  Since $(a_{i-1}, a_i)$ lying in $R(G_{V_{i'}})$ means there is a path $p_{i'}$ in $G_{V_{i'}}$ from vertex $a_{i-1}$ to vertex $a_i$, the concatenation $p_{1'}. \cdots .p_{m'}$ gives a path from $x$ to $y$ in the graph $G$, i.e.\ $(x,y) \in R(G)$.

For the converse, suppose $(x,y) \in R(G)$.  This means there is  a path $p$ from $x$ to $y$ in the graph $G$.  Suppose $p_1, \cdots, p_k$ are paths in $G_{V_{1'}}, \cdots, G_{V_{k'}}$, respectively, that satisfy condition (U) for $p$.  That $p_i$ is a path in $G_{V_{i'}}$ means there is some element $(b_{i-1}, b_i)$ in $R(G_{V_{i'}})$ such that $p_i$ has $b_{i-1}, b_i$ as the starting point and end point, respectively.  Since $p_1$ starts at $x$ and $p_k$ ends at $y$, it follows that $(x,y)$ lies in the transitive closure of $\cup_{i=1}^k R(G_{V_{i'}})$ which in turn is contained in the transitive closure of $\cup_{i=1}^{n} R(G_{V_i})$.  Hence we are done.
\end{proof}

The next proposition explains how the common WD graph for a collection of sequences can be glued together from the common WD graphs of different restrictions of those sequences.

\begin{prop}\label{prop:gluing-common}
Let $X$ be a finite set.  Suppose $s^{(1)}, \cdots, s^{(p)}$ are simple sequences in $X$ and $V_1, \cdots,V_n$ are nonempty subsets of $X$.  Let $G$ denote the common WD graph of the sequences $s^{(1)}, \cdots, s^{(p)}$.  For each $1 \leq i \leq n$, let $G_i$ denote the common WD graph of the sequences $m_{V_i}(s^{(1)}), \cdots, m_{V_i}(s^{(p)})$.  If $G$ and $V_1, \cdots, V_n$ satisfy condition (U), then $G$ is the transitive reduction of
\begin{equation}\label{eq:gluingformula}
 \bigcup_{i=1}^n R(G_i).
\end{equation}
\end{prop}

\begin{proof}
By Proposition \ref{prop:WDgluing-01}, it suffices to show that $G_i = G_{V_i}$ for each $1 \leq i \leq n$.  To this end, fix an index $i$.  Since both $G_i$ and $G_{V_i}$ are  transitive reductions of DAGs with vertex set $V_i$,  it suffices to show that they have the same reachability.

Suppose there is a path from vertex $x$ to vertex $y$  in $G_i$.  Then $x,y$ both lie in $V_i$ and, for all $1 \leq j \leq p$, both  $x, y$  appear in $s^{(j)}$ with $x$ appearing as an earlier term than $y$.  Hence there must be a path from $x$ to $y$  in the graph $G$, and subsequently there is a path from $x$ to $y$ in the graph  $G_{V_i}$. 

Conversely, suppose there is a path from vertex $x$ to vertex $y$   in $G_{V_i}$.  This means $x, y$ both lie in $V_i$ and, for all $1 \leq j \leq p$, both $x, y$ appear in $s^{(j)}$ with $x$ appearing before $y$.  It follows that for all $1 \leq j \leq p$, both $x, y$ appear in $m_{V_i}(s^{(j)})$ with $x$ appearing before $y$, implying there is a path from $x$ to $y$ in $G_i$, and we are done.
\end{proof}

Theorem \ref{thm:main-01} below answers the following question:

\begin{center}
\fbox{\parbox{0.9\linewidth}{\centering
\emph{Suppose we run Algorithm \ref{algo:HC-v2} on a set of sequences and obtain an output $\mathcal{C} = \{ C_1, \cdots, C_l\}$.  Fix any $1 \leq i \leq l$.  For every $G \in C_i$, how is $G$ related to the common WD graph $G'$ of all the sequences that are consistent with $G$?}
}}
\end{center}
We show that when the elements of $C_i$ satisfy a ``non-overlap'' condition, the answer is that $G$ and $G'$ agree.

\begin{thm}\label{thm:main-01}
Let $X$ be a finite set and  $s^{(1)}, \cdots, s^{(p)}$ simple sequences in $X$.    Fix a positive integer $r$, and suppose that, with $t=100$ and $J=X$ as the inputs for Algorithm \ref{algo:HC-v2},  the output $\mathcal{C}$ contains 
\[
\overline{C} = \{ G_1, \cdots, G_m\},
\]
where the $G_i$  are quasi-skeleton WD graphs.  Suppose, in addition, that the following condition holds:
\begin{itemize}
    \item[(D)] For each $1 \leq i \leq p$, the sequence $s^{(i)}$ is consistent with exactly one of $G_1, \cdots, G_m$.
\end{itemize}
For each $1 \leq i \leq m$, define
\[
Q_i = \{ s^{(j)} : s^{(j)} \text{ is consistent with } G_i \}.
\]
Then $G_i$ coincides with the common WD graph of the sequences in $Q_i$.
\end{thm}

Note that condition (D) implies that the $Q_i$ form a partition of the set $\{ s^{(i)} : 1 \leq i \leq p\}$ of all the sequences.

 Conceptually, Theorem \ref{thm:main-01} says the following: when a set of data produces a collection of wiring diagram graphs $G_1, \cdots, G_m$ summarizing the concepts within the data, if the data points corresponding to distinct $G_i$ form clusters (the $Q_i$) that are disjoint, then each $G_i$ agrees with the common WD graph extracted from its own cluster.

\begin{proof}
Fix an index $1 \leq i \leq m$.  Let $G_i^\ast$ denote the common WD graph of the elements of $Q_i$.  Our goal is to show $G_i^\ast = G_i$.  Since $G_i, G_i^\ast$ are both transitive reductions of DAGs, it suffices to show they have the same reachability.

Suppose there is a path from vertex $x$ to vertex $y$ in $G_i$.  This means that for all $s \in Q_i$, both $x, y$ appear in $s$ with $x$ as the earlier term.  Hence there must be a path from $x$ to $y$ in $G_i^\ast$ by the definition of $G_i^\ast$.

Conversely, suppose there is a path from vertex $x$ to vertex $y$ in $G_i^\ast$.  This means that $x$ appears before $y$ in every sequence $s \in Q_i$, and so $(x,y) \in R(\stg (s))$ for all $s \in Q_i$. For the sake of contradiction, suppose there is no path from $x$ to $y$ in the graph $G_i$.  Let $\widetilde{G_i^+}$ denote the directed graph obtained by adding an arrow from $x$ to $y$ in the graph $G_i$.  For all $s \in Q_i$, by the construction of $Q_i$ we know $s$ is consistent with $G_i$; then by \cite[Corollary 5.15]{LJ-D2WD}, we know there is a morphism $\stg (s) \to G_i$ in the category $\mathcal{R}(J)$, and so $R(G_i) \subseteq R(\stg (s))$.  Now 
$R(\stg (s))$ is a poset that contains the subset $R(G_i)$ as well as the element $(x,y)$.  This means that adding an arrow from $x$ to $y$ in the graph $G_i$ would not yield a cycle (we are assuming $x \neq y$), i.e.\ $\widetilde{G_i^+}$ is a DAG.  Let $G_i^+$ denote the transitive reduction of $\widetilde{G_i^+}$.  Note that $R(G_i) \subsetneq R(G_i^+)$ because $(x,y) \in R(G_i^+) \setminus R(G_i)$.

Let $\overline{C}^+$ denote the set obtained by replacing $G_i$ with $G_i^+$ in $\overline{C}$, i.e.\ 
\[
\overline{C}^+ = (\overline{C} \setminus \{ G_i \}) \cup \{ G_i^+\}.
\]
Note that $|\overline{C}^+| \leq r$ since $|\overline{C}|\leq r$.  Also, for every $s \in Q_i$, since $G_i^\ast$ is the common graph of the sequences in $Q_i$ and there is a path from $x$ to $y$ in $G_i^\ast$, $x$ must appear before $y$ in the sequence $s$; together with the fact that $s$ is consistent with $G_i$ (by the definition of $Q_i$), it follows that $s$ is consistent with $\widetilde{G_i^+}$, and hence consistent with $G_i^+$.  Hence
 \[
       \frac{\left| \{ i \in \mathbb{Z} : 1 \leq i \leq p, S_i \cap \overline{C}^+ \neq \varnothing \}\right|}{p} \geq t\% =100\%.
       \]
That is, conditions (i) and (iii) in line \ref{alg:HC-v2:line-3} of Algorithm  \ref{algo:HC-v2} are satisfied for  $\overline{C}^+$.  We now show condition (ii) also holds for $\overline{C}^+$.

By the construction of $\overline{C}$, if there is a morphism in the category $\mathcal{R}(J)$ between two elements of $\overline{C}^+$, it must be between $G_i^+$ and some $G_j$ where $j \neq i$.  We divide into two cases:
\begin{itemize}
    \item[(i)] There is a morphism $G_j \to G_i^+$: this implies   $R(G_i) \subsetneq R(G_i^+) \subseteq R(G_j)$, contradicting condition (ii) in line 3 of Algorithm  \ref{algo:HC-v2}.
    \item[(ii)] There is a morphism $G_i^+ \to G_j$: this implies  $R(G_j)  \subseteq R(G_i^+)$.  Note that  every $s \in Q_i$ is consistent with $G_i^\ast$ by construction of $G_i^\ast$, and so every $s \in Q_i$ is consistent with $G_j$, contradicting assumption (D).
\end{itemize}
Therefore, there are no morphisms between any two elements of $\overline{C}^+$ in the category $\mathcal{R}(J)$, i.e.\ $\overline{C}^+$ satisfies condition (ii) in line \ref{alg:HC-v2:line-3} of Algorithm  \ref{algo:HC-v2}, and so $\overline{C}^+$ is one of the vertices in the graph $\mathcal{G}$ in line \ref{alg:HC-v2:line-4} of Algorithm  \ref{algo:HC-v2}.

Since $R(G_i) \subsetneq R(G_i^+)$, there is an arrow  $\overline{C}^+ \to \overline{C}$ in  $\mathcal{G}$.  This, however, contradicts the condition that any element in the output $\mathcal{C}$ of Algorithm \ref{algo:HC-v2} (in particular, the element $\overline{C}$)   has no arrows pointing to it in $\mathcal{G}$.  Hence the assumption that there is no path from $x$ to $y$ in $G_i$ must be false, and we are done.
\end{proof}

\section{Example: gluing wiring diagrams from figure skating data}\label{sec:maineg}

In this section, we give an extended example  that showcases how the theory in Section \ref{sec:theory-gluingWDs} can be implemented in practice.

\begin{eg}\label{eg:main-01}
Let $X= \{e_1, \cdots, e_8\}$ where the $e_i$ are as in Section \ref{sec:intro-FSjumps}.  For this example, let us assume that we have sequences $s^{(1)}, \cdots, s^{(p)}$ for some positive integer $p$, extracted from  videos of figure skaters performing jumps, and that the sequences were extracted using classifiers for the four pairs of features in \ref{para:intro-testingth}, with respect to staged ordering as in \ref{para:eventseq_construction}.  We also assume the sequences represent  ``perfect data'' in the sense described below.   Let $S = \{ s^{(i)} : 1 \leq i \leq p \}$ denote the set of all the sequences.  
\begin{enumerate}
    \item If a sequence $s^{(i)}$ comes from a video of a skater performing a particular type of jump, then $\stg (s^{(i)})$ is a flattening of the WD graph for that jump as in \eqref{eq:jumpWDs}.  
    In particular, $s^{(i)}$ is consistent with the WD graph for that jump.
    \item Given the ideal WD graph of any jump as in \eqref{eq:jumpWDs}, and given any flattening $H$ of that WD graph,  there is at least one sequence $s$ in $S$ such that $\stg (s)$ coincides with $H$.  Loosely speaking, this means the set $S$ contains all the possible sequences that can  occur under staged ordering.    
\end{enumerate}  
Now run  Algorithm  \ref{algo:HC-v2} on the sequences in $S$ with $X$ as above and taking $J$ to be 
\[
V_1 = \{ e_1, e_2, e_5, e_6, e_7, e_8\},
\]
 $r=5$, and $t=100$.  Since we are assuming $S$ consists of perfect data, the output $\mathcal{C}$ contains  a  set whose  elements are  graphs corresponding to toe loop,  salchow, loop, and axel, together with an additional graph $H_{fl}$ that represents the commonality between flip and lutz (since the events in $V_1$ are not sufficient to differentiate between flip and lutz, which are distinguished by inside/outside edge  at takeoff) - see Figure \ref{fig:fsWDGs-V1}.  Note that condition (D) in Theorem \ref{thm:main-01} is satisfied, and so we can  partition $S$ into disjoint subsets that correspond to toe loop, loop, salchow, axel, and one subset $Q_{fl}$ that contains all the sequences that come from either flip or lutz videos.   Theorem \ref{thm:main-01} now says $H_{fl}$ is precisely the common WD graph of the sequences in $Q_{fl}$ with respect to the set $V_1$.

\begin{center}
  \includegraphics[scale=0.55]{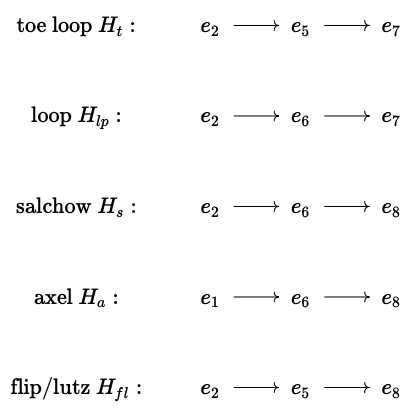}
  \captionof{figure}{WD graphs for jumps using events in $V_1$.}
  \label{fig:fsWDGs-V1}
\end{center}

Next, run  Algorithm  \ref{algo:HC-v2} again, with the same $X$ as before but on the sequences in $Q_{fl}$ and with $J$ taken to be
\[
 V_2 = \{ e_1, e_2, e_3, e_4, e_5, e_6\},
\]
$r=2$, and $t=100$.  The output $\mathcal{C}$ contains   a set with two elements $H'_f, H'_{lz}$ that correspond to flip and lutz, respectively.  Since we are assuming $S$ contains perfect data, condition (D) in Theorem \ref{thm:main-01} is again satisfied; writing $Q_f, Q_{lz}$ to denote the subsets of $Q_{fl}$ consisting of sequences that are consistent with $H'_f, H'_{lz}$, respectively, Theorem \ref{thm:main-01} says $H'_f, H'_{lz}$ coincide with the common WD graphs of sequences in $Q_f, Q_{lz}$, respectively.

\begin{center}
  \includegraphics[scale=0.55]{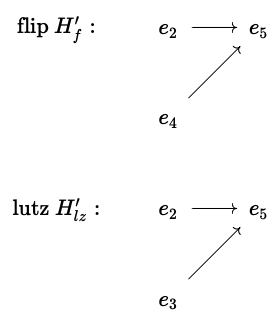}
  \captionof{figure}{WD graphs for flip and lutz using events in $V_2$.}
  \label{fig:fsWDGs-V2}
\end{center}

Finally, note that under staged ordering,  every sequence $s$ in $S$ has the property that every path in $\stg (s)$ is a concatenation of a path in $(\stg (s))_{V_2}$ and a path in $(\stg (s))_{V_1}$.  That is, for any $s \in S$, the graph $\stg (s)$ and $V_1, V_2$ together satisfy condition (U) in \ref{para:pathglueingcond}.  Therefore, by Proposition \ref{prop:gluing-common}, we can recover the complete WD graphs for flip and lutz shown in \eqref{eq:jumpWDs} using the gluing formula \eqref{eq:gluingformula}:
\begin{itemize}
    \item The WD graph for flip, written in terms of the full set of events $X$, is the transitive reduction of $R(H_{fl}) \cup R(H'_f)$ which is 
    \begin{equation}\label{eq:flip-recovered}
    \xymatrix{
    e_2 \ar[r] & e_5 \ar[r] & e_8 \\
    e_4 \ar[ur] & & 
    } 
    \end{equation}
    \item The WD graph for lutz, written in terms of the full set of events $X$, is the transitive reduction of $R(H_{fl}) \cup R(H'_{lz})$ which is 
    \begin{equation}\label{eq:lutz-recovered}
    \xymatrix{
    e_2 \ar[r] & e_5 \ar[r] & e_8 \\
    e_3 \ar[ur] & & 
    } 
    \end{equation}
\end{itemize}
Note that the WD graphs obtained above via gluing agree with the ideal WD graphs for flip and lutz in \eqref{eq:jumpWDs}. 
\end{eg}

\paragraph[Comparison with results in \ref{para:HCactualresults}]  The WD graphs that are produced under the first iteration (Figure \ref{fig:fsWDGs-V1}) and the second iteration (Figure \ref{fig:fsWDGs-V2}) coincide with the WD graphs obtained from our actual data in Tables \ref{tab:WDfrom1stHC} and \ref{tab:WDfrom1stHC-2}.  Therefore, if we apply the gluing procedure outlined in the last part of Example \ref{eg:main-01} to the actual data in Section \ref{para:HCactualresults}, we would obtain WD graphs \eqref{eq:flip-recovered} and \eqref{eq:lutz-recovered} which are exactly the ideal graphs for those jumps in \eqref{eq:jumpWDs}.  This demonstrates that when applied to real data, the theory  for gluing wiring diagrams from iterative Hasse clustering developed in Section \ref{sec:theory-gluingWDs}  can produce wiring diagrams with  greater detail than those produced by a single application.

\section{Concluding remarks}\label{sec:conclusion}

\paragraph[Summary] Hasse clustering is an algorithm developed in \cite{LJ-D2WD} for extracting graphical representations of concepts - called wiring diagrams - from sequential data.  In this article, we considered the problem of applying Hasse clustering  to a dataset iteratively  and synthesizing the results.  We developed a theory for gluing wiring diagrams produced during  iterative Hasse clustering, and demonstrated the implementation  of this theory using 3D motion-capture data in figure skating.  We trained specialist algorithms to recognize technical features of figure skating jumps; these specialist algorithms then allowed us to convert videos of figure skating jumps to sequences of events. Hasse clustering then sorted these sequences into clusters with high accuracy, and our theory of gluing allowed us to recover the full wiring diagrams of these clusters.

\paragraph[Further applications]  Although the primary motivation for iterative Hasse clustering in this article was to overcome the combinatorial complexity of our algorithm, our method could also be used to synthesize datasets collected and analyzed in separate batches at different points in time.

In addition, despite the focus on figure skating data in this article, the framework of analyzing data using wiring diagrams and Hasse clustering  applies to any other context in which data can be translated to sequences of events.  In sports, for example, analyzing sequences of movements can be useful for training athletes  and developing on-court strategies in basketball \cite{wang2024basketball}.  Depending on whether the events of interest that feed into Hasse clustering are lower-level body movements or higher-level plays, our method could potentially be used to classify and detect low-level actions or high-level strategies in a game.

More generally, human comprehension and learning rely on action segmentation, which is the fundamental cognitive mechanism of extracting distinct events from a continuous stream of dynamic behavior \cite{baldwin2008segmenting}.  Our method for extracting patterns in sequences for further analysis can be layered on top of action segmentation in the design of reasoning and learning frameworks for autonomous systems.

\paragraph[Data and code availability]  The dataset used in this study, FS-Jump3D, is publicly available from its original repository \cite{tanaka2024fsjump3d}. The complete implementation of our framework is available through two GitHub repositories. The repository Figure-Skating-Jump-Analysis-Pipeline \cite{jafari2026figureskating} contains the full pipeline for processing the FS-Jump3D dataset, extracting the four jump features, training the classifiers, generating symbolic event sequences, and reproducing the experiments described in this paper. The repository Hasse-Clustering \cite{jafari2026hasseclustering} contains the implementation of the new and legacy Hasse clustering algorithm used to generate the wiring diagrams and clustering results. Together, these repositories provide a complete implementation of the methodology presented in this article and enable researchers to reproduce the reported experiments and apply the framework to other sequential datasets.

\paragraph[Acknowledgments] This material is based upon work supported by the Air Force Office of Scientific Research under award 
number FA9550-24-1-0268.

\end{multicols}

\bibliography{refsMR2}{}
\bibliographystyle{plain}

\end{document}